\documentclass[preprint,sts]{imsart}
  \setattribute{journal}{name}{}

\usepackage{amsthm,amsmath,natbib}
  \numberwithin{equation}{section}
\RequirePackage[colorlinks,citecolor=blue,urlcolor=blue]{hyperref}
\usepackage{booktabs}
\usepackage{color}
  \definecolor{Gray}{gray}{0.9}
\usepackage{colortbl,easytable,hhline}
\usepackage{graphicx,epstopdf}
  \graphicspath{{./figs/}}
\usepackage{subcaption}
\usepackage{tikz}
  \usetikzlibrary{bayesnet}

\startlocaldefs

\newif\ifblog
\newif\iftex
\blogfalse
\textrue

\usepackage{amssymb,amsthm,bm}
\usepackage{ulem}
\def\emph#1{\textit{#1}}

\newenvironment{blue}{\color{blue}}{}

\def\bA{\mathbf{A}}

\def\bG{\mathbf{G}}

\def\bK{\mathbf{K}}

\def\bX{\boldsymbol X}

\def\ba{\mathbf{a}}

\def\be{\mathbf{e}}

\def\beps{{\boldsymbol\epsilon}}

\def\be{\mathbf{e}}

\def\bq{\mathbf{q}}

\def\bs{\mathbf{s}}

\def\bu{\mathbf{u}}

\def\bx{\mathbf{x}}
\def\by{\mathbf{y}}
\def\bz{\mathbf{z}}

\def\bbE{\mathbb{E}}

\def\bbP{\mathbb{P}}
\def\bbQ{\mathbb{Q}}
\def\bbR{\mathbb{R}}

\def\cA{\mathcal{A}}

\def\cE{\mathcal{E}}

\def\cH{\mathcal{H}}

\def\cL{\mathcal{L}}

\def\cN{\mathcal{N}}

\def\cR{\mathcal{R}}
\def\cS{\mathcal{S}}

\def\cX{\mathcal{X}}
\def\cY{\mathcal{Y}}
\def\cZ{\mathcal{Z}}

\def\cov{{\sf cov}}

\def\cf{{\it cf.}}

\def\Ex{{\bbE}}

\def\eg{{\it eg.}}
\def\eps{{\epsilon}}

\def\iid{{\it i.i.d.}}

\def\hmu{\widehat{\boldsymbol\mu}}

\def\hSigma{\widehat{\boldsymbol\Sigma}}

\def\ones{\mathbf{1}}

\def\reals{{\bbR}}

\def\tr{{\sf tr}}

\def\unif{{\sf unif}}

\newtheorem{theorem}{Theorem}[section]
\newtheorem{lemma}[theorem]{Lemma}
\newtheorem{definition}[theorem]{Definition}

\newtheorem{example}[theorem]{Example}

\def\hby{\widehat{\by}}
\def\HS{{\sf HS}}

\endlocaldefs

\begin{document}

\begin{frontmatter}

\title{On conditional parity as a notion of non-discrimination in machine learning}
\runtitle{Conditional parity}

\begin{aug}
\author{\fnms{Ya'acov} \snm{Ritov}\ead[label=e1]{yritov@umich.edu}},
\author{\fnms{Yuekai} \snm{Sun}\ead[label=e2]{yuekai@umich.edu}},
\and
\author{\fnms{Ruofei} \snm{Zhao}\ead[label=e3]{rfzhao@umich.edu}}

\runauthor{Ritov, Sun, Zhao}

\affiliation{University of Michigan}

\end{aug}

\begin{abstract}
We identify conditional parity as a general notion of non-discrimination in machine learning. In fact, several recently proposed notions of non-discrimination, including a few counterfactual notions, are instances of conditional parity. We show that conditional parity is amenable to statistical analysis by studying randomization as a general mechanism for achieving conditional parity and a kernel-based test of conditional parity.
\end{abstract}



\end{frontmatter}

\section{Non-discrimination in machine learning}

As automated decision systems permeate our world, the problem of implicit biases in these systems have become more serious. Machine learning algorithms are routinely used to make decisions in credit, criminal justice, and education, all of which are domains protected by anti-discrimination law. Although automated decision systems seem to eliminate the biases of a human decision maker, they may perpetuate or even exacerbate biases in the data.

For example, consider an advertising platform which uses demographic information of visitors to a website to decide which credit card offers to show first-time visitors. If the system is trained on historical data where minority visitors were given less advantageous offers, the system may steer similar visitors to less advantageous offers, which is illegal \citep{Steel2010webs}.

In response, the scientific community has proposed several formal definitions of non-discrimination and various approaches to ensure algorithms are non-dis\-criminatory. Unfortunately, the myriad of definitions and approaches hinders the adoption of this work by practitioners: they must choose from the growing list of definitions and approaches, and there is often no clear choice.

In light of this plethora of definitions, we identify a general notion of non-discrimination in Section 2 that not only includes many recently proposed definitions but also suggests new definitions. In Sections 3 and 4, we study randomization as a general mechanism for achieving conditional parity and a kernel-based test of conditional parity. Finally, in Section 5, we apply this test to determine whether insurance companies charge higher premiums to insure cars in minority neighborhoods.

\section{Conditional parity: a notion of non-discrimination}

Intuitively, any claim of discrimination or non-discrimination depends on a comparison: a comparison between the outcome of two groups that differ only by a sensitive attribute. For instance, a claim of gender discrimination by a female employee of a consulting firm implies she was treated differently from male employees of the company in her position. Here the two groups are female and male employees that share her position. By only comparing herself to male employees in her position, she is implicitly permitting the firm to treat male employees in other positions differently. In other words, the firm is allowed to discriminate on an employee's position (\eg\ paying senior employees higher salaries). We see that {\it in order to fully specify the two groups, we must not only specify the protected attribute (\eg\ gender) but also specify the discriminatory attribute (\eg\ position)}.

\begin{definition}[conditional parity (CP)]
\label{def:conditional-parity}
A random variable $\bx$ satisfies pa\-rity with respect to $\ba$ conditioned on $\bz = z$ if the distribution of $\bx\mid\ba,\{\bz = z\}$ is constant in $\ba$:
\[
\cL(\bx\mid\ba = a,\bz = z) =\cL(\bx\mid\ba = a',\bz = z)\text{ for any }a,a'\in\cA.
\]
Similarly, $\bx$ satisfies parity with respect to $\ba$ conditioned on $\bz$ (without specifying a value of $\bz$) if it satisfies parity with respect to $\ba$ conditioned on $\bz = z$ for any $z\in\cZ$.
\end{definition}

In terms of independence, conditional parity is $\bx\perp\ba\mid\{\bz = z\}$. Table \ref{tab:segments} is a graphical representation of the groups in the running gender discrimination example. As we shall see, many existing notions of non-discrimination such as demographic parity, equalized odds, equalized opportunity, and counterfactual fairness are all instances of CP. We remark that the definition of CP is {\it invariant under post-processing}: if $\bx$ satisfies CP with respect to $\ba$ conditioned on $\bz = z$, then so does $f(\bx)$ for an arbitrary function $f$. This is especially desirable because it leads to a simple way of eliminating bias in machine learning algorithms.

\newcolumntype{g}{>{\columncolor{Gray}}l}
\begin{table}
  \centering
  \caption{A graphical representation of the groups implicit in the gender  discrimination example.
  }
  \vspace{6pt}
  \begin{tabular}{l|g|g|}
  \rowcolor{white} \multicolumn{1}{l}{} & \multicolumn{1}{l}{$\ba=$ female} & \multicolumn{1}{l}{$\ba=$ male} \\[12pt] \hhline{~|-|-|}
  \rowcolor{white} $\bz=$ analyst & & \\[12pt] \hhline{~|-|-|}
  $\bz=$ associate & & \\[1em] \hhline{~|-|-|}
  \rowcolor{white} \vdots           & & \\[12pt] \hhline{~|-|-|}
  \rowcolor{white} $\bz=$ VP        & & \\[12pt] \hhline{~|-|-|}
  \end{tabular}
  \label{tab:segments}
\end{table}

The intuition of identical conditional distributions that Definition \ref{def:conditional-parity} formalizes extends easily to yield approximate notions of non-discrimination. To keep things simple, we assume $\ba$ is discrete.

\begin{definition}[$\eps$-conditional parity]
\label{def:eps-conditional-parity}
Let $d$ be a metric on distributions.\footnote{Formally, the metric must satisfy $d(\bbP,\bbQ) = d(\bbQ,\bbP)$ and $d(\bbP,\bbQ) \ge 0$ for any pair of distributions $\bbP$ and $\bbQ$ as well as $d(\bbP,\bbP) =0 $ for any distribution $\bbP$.}
A random variable $\bx$ satisfies $\eps$-conditional parity with respect to $\ba\in\cA$ conditioned on $\bz = z$ if
\[
{\textstyle\max_{a,a'\in\cA}}d(\cL(x\mid\ba = a,\bz = z), \cL(x\mid\ba = a',\bz = z)) \le \eps\text{ for any }a,a'\in\cA.
\]
\end{definition}

To wrap up, we compare CP to two other notions of non-discrimination: functional blindness and individual fairness.

\begin{definition}[functional blindness]
\label{def:functional-bindness}
A decision rule $\delta:\cA\times\cZ\to\cX$ satisfies functional blindness with respect to $a$ iff
\[
\delta(a,z) = \delta(a',z)\text{ for any }a,a'\in\cA\text{ and }z\in\cZ.
\]
In other words, the decision rule has no functional dependence on the protected attribute.
\end{definition}

Functional blindness, also known as fairness through unawareness, is a rudimentary but widely used notion of non-discrimination. Although intuitive, it is a weak notion that is easily circumvented because it does not rule out implicit dependence of the decision on the protected attribute.

For decades, insurance companies have charged drivers in predominantly minority neighborhoods higher premiums than drivers in majority white neighborhoods. Although insurers have justified their pricing by citing a  higher risk of accidents in minority neighborhoods, consumer advocates suspect the practice is merely a way around laws that ban discriminatory rate-setting: a driver's zip code is a good proxy for his or her race in segregated areas.

We remark that functional blindness implies parity conditioned on $z$. However, if $z$ includes attributes that are proxies for the protected attribute (\eg\ zip code is a proxy for race in the preceding example), enforcing CP is vacuous. After all, by including an attribute in $z$, we are allowing the decision rule to discriminate based on it.

\begin{definition}[individual fairness \citep{Dwork2012Fairness}]
\label{def:lipschitz-fairness}
Let $d$ be a metric on distributions and $D$ be a metric on the space of individuals $\cX$. A (possibly randomized) decision rule $\delta$ satisfies $\eps$-individual fairness iff it is $\eps$-Lipschitz in $x$:
\[
d(\cL(\delta(x_1)),\cL(x_2)) \le \eps D(x_1,x_2)\text{ for any }x_1,x_2\in\cX.
\]
\end{definition}

Individual fairness is based on the principle that two similar individuals should be treated similarly by the decision rule. The precise definition of individual fairness depends crucially on the choice of the metrics $d$ and $D$. \cite{Dwork2012Fairness} suggest the metrics be chosen by a regulatory body or proposed by civil rights organizations and left open to discussion and continual refinement.

Both CP and individual fairness formalize the intuition that similar individuals should be treated similarly. In CP, similar individuals are those that share discriminatory attributes. In individual fairness, similar individuals are determined by the choice of the the metric on individuals. Although \cite{Dwork2012Fairness} does not distinguish between disriminatory and protected attributes, it is possible to encode the distinction into the choice of metric on $\cX$.

\subsection{Demographic parity and equalized odds}

In this subsection, we describe several factual (as opposed to counterfactual) notions of non-discrimination and show that they are instances of CP.

\begin{definition}[demographic parity (DP)]
  The outcome $\bx$ satisfies demographic parity if
  \[
  \cL(\bx \mid \ba=a) = \cL(\bx \mid \ba=a')\text{ for any }a,a'\in\cA.
  \]
\end{definition}

As we can see, there is no discriminatory attribute in DP, and it is required that individuals from the group $\{\ba=a\}$ has to be treated equally as individuals from the group $\{\ba=a'\}$. Although in general, DP seems too coarse a notion of non-discrimination, there are some scenarios where it is suitable. For example, in the allocation of public resources, DP is a fitting notion of non-discrimination. A concrete example is public secondary school admission. Due to the public service nature of public secondary education, parents should be allowed to send their children to any school in their neighborhood, regardless of their background. In reality, such goal is often attained by lottery, meaning that random selections in the pool of applicants are made.

\begin{example}[War on Drugs]
\label{ex:war-on-drugs}
According to the American Civil Liberties Union (ACLU), ``an African American adult is 2.8 times as likely to have a misdemeanor cannabis charge filed against him or her than does an Anglo American adult'' in Washington State \citep{Jensen2016Field}. By comparing the likelihood of being charged without stratifying the population (\eg\ by prevalence of cannabis consumption), the ACLU is claiming the War on Drugs violates DP. This is an example where DP is not a suitable notion of non-discrimination: the disparity between the likelihood of being charged may be due to disparities between prevalence of cannabis consumption. Thus it is incorrect to conclude the targeting of African Americans by law enforcement from violation of DP.
\end{example}

To avoid the problems of DP, we first segment the population by certain discriminatory attributes (\eg\ prevalence of cannibis consumption in the War on Drugs example) and then apply DP to each segment of the population. This led us to the notion of CP. In supervised learning, a natural instantiation of CP is {\it equalized odds}, which appeared in \cite{Hardt2016} and \cite{Zafar2017}.

\begin{definition}[equalized odds (EO) \cite{Hardt2016}]
A prediction $\hby$ of $\by\in\cY$ satisfies {\textit equalized odds} with respect to protected attribute $\ba$ and outcome $\by$ if
  \[
  \cL(\hby \mid \ba=a, \by=y) = \cL(\hby \mid \ba=a', \by=y)\text{ for any }a,a'\in\cA\text{ and }y\in\cY.
  \]
\end{definition}

In terms of CP, EO is equivalent to the prediction $\hby$ satisfying parity with respect to the protected attributed $\ba$ conditioned on the outcome $\by$. In other words, EO requires the individuals with the same $\by$ but differing $\ba$ to be treated equally. If both $\by$ and $\hby$ are binary, then in standard terminology, EO means that the probabilities of false alarms and the detection probabilities are the same under all possible values of $\ba$. Note that typically the optimal ROC curve depends on the value of $\ba$, and oftentimes, we sacrifice some efficiency to achieve EO through randomization \cite{Hardt2016}.

\begin{example}[Example \ref{ex:war-on-drugs} continued]
In the War on Drugs example, $\hby$ is whether an individual is charged, $\ba$ is an individual's race, and $\by$ may be whether an individual consumes cannabis. If there is a discrepancy between the prevalence of cannibis consumption among African and Anglo Americans, then EO is more suitable notion of non-discrimination in law enforcement. Since $\by$ depends on $\ba$, even perfect prediction $\hby = \by$ violates DP, but it is hardly discriminatory for law enforcement to charge anyone who consumes cannabis with a misdemeanor. On the other hand, it is easy to check that the perfect predictor satisfies EO.
\end{example}

In some applications, one of the outcomes $y\in\cY$ is considered ``advantaged''. For example, consider the use of historical repayment data to predict default. If the historical data contains biases against minority groups, the prediction system may echo the bias in its predictions. A possible relaxation of EO is to require people who will not default to have equal chance of getting a loan, regardless of their race.

\begin{definition}[equal opportunity \cite{Hardt2016}]
  Let $\by=1$ be the ``advanta\-ged'' outcome. A prediction $\hby$ satisfies {\textit equalized opportunity} with respect to protected attribute $\ba$ and outcome $\by$ if
  \[
  \cL(\hby \mid \ba=a, \by=1) = \cL(\hby \mid \ba=a', \by=1)\text{ for any }a,a' \in \cA.
  \]
\end{definition}

It is easy to see how equalized odds leads to the more general notion of CP. The key idea of comparing segments of the population that share discriminatory attributes but differ in the protected attribute is clear. In classification, there is a natural discriminatory attribute: the outcome $\by$. However, it is worth considering other ways of segmenting the population, even in supervised learning.

\begin{example}[gender bias in UC Berkeley admissions \cite{Bickel1977Sex}]
In the autumn of 1973, the graduate division of UC Berkeley admitted 44\% of male applicants but only 35\% of female applicants, prompting allegations of gender bias in the admissions process. However, adjusting the admissions outcome by department reveals a ``small but statistically significant bias in favor of women''. \cite{Bickel1977Sex} concluded that women tended to apply to highly competitive departments, which admit a smaller percentage of applicants, while men tended to apply to less competitive departments. In this example, including department as a discriminatory attribute leads to a qualitatively different conclusion.
\end{example}

As an aside, this example also shows that CP generally does not imply DP. Even if the admission rates of male and female applicants are identical in all departments, the admission rates to the graduate division may still differ if male and female applicants apply to departments at different rates. Conversely, even if the admission rates of male and female applicants to the graduate division are identical, the admission rates to each department may differ. This reveals another problem of DP: it permits disparate treatment within segments of the population as long as the disparities ``cancel out'' on average. Although this is rare in practice, we point it out to emphasize CP and DP are generally incomparable.

Finally, to highlight the generality of CP, we describe an application of CP in representation learning. In machine learning, feature or representation learning is the task of learning a transformation of raw data to a feature vector that is amenable to machine learning algorithms. By letting $\bx\in\reals^d$ be the learned feature vector, CP readily leads to a notion of non-discrimination in representation learning:
\[
\cL(\bx\mid\ba = a,\bz = z) =\cL(\bx\mid\ba = a',\bz = z)\text{ for any }a,a'\in\cA\text{ and }z\in\cZ.
\]
As we shall see, this notion of non-discrimination has been implicitly used in natural language processing (NLP).

To wrap up, we describe a post-processing method that returns a new feature vector that satisfies CP. To keep things simple, we assume $\bx\in\reals^d$, $\ba\in\reals^{k_1}$, and $\bz\in\reals^{k_2}$ are jointly Gaussian. Without loss of generality, let
\[
\beps^T = \bx^T - \bz^TB - \ba^T\Gamma,
\]
where $B\in\reals^{k_2\times d}$ and $\Gamma\in\reals^{k_1\times d}$ are chosen so that $\Ex\bigl[\beps\mid\bz,\ba\bigr] = 0$. Rearranging, we have
\[
\bx^T = \ba^TB + \bz^T\Gamma + \beps^T.
\]
A new feature vector $\bx'$ that satisfies $\bx'\perp\ba\mid\bz$ is
\[
\bx' = (I_d - P_B)\bx,\quad P_B = B^T(B^T)^\dagger.
\]
One way to estimate $\cR(B^T)$ is to select a subset of feature vectors that are similar in $\bz$ and compute their principal components.
This is essentially the approach proposed by \cite{Bolukbasi2016} to remove gender bias in word embeddings.

\begin{example}[debiasing word embeddings \citep{Bolukbasi2016}]
A word embedding is a representation of words by vectors in $\reals^d$. Word embeddings enable machine learning algorithms to reason semantically by performing arithmetic operations on the word embeddings; \eg
\[
{\sf grandfather} - {\sf man} + {\sf woman} = {\sf grandmother}.
\]
They are learned from text corpus and inherit implicit biases in the texts. For example, according to the the popular word2vec embedding, which is trained on a corpus of Google News articles, we have
\[
{\sf engineer} - {\sf man} + {\sf woman} = {\sf homemaker}.
\]
To remove gender bias in word embeddings, \cite{Bolukbasi2016} propose a method that identifies a gender subspace and projects the embedding onto the orthocomplement of the gender subspace to obtain a debiased word embedding. To identify the gender subspace, the method takes pairs of words whose meanings differ only in gender (\eg\ (actor, actress), (father, mother)) and estimates the principal compoments of the pairwise differences.
\end{example}

By the invariance of CP under post-processing, the output of a machine learning algorithm based on features that satisfy CP inherits the property.  This suggests using non-discriminatory features as a simple approach to eliminating bias in machine learning algorithms.

\subsection{Counterfactual notions of non-discrimination}

In order to work with counterfactuals, we must impose some modeling assumptions on the data generating process. In the rest of this subsection, we assume the data is generated by a structural equations model (SEM). A SEM consists of (i) a set of random variables, (ii) a set of (deterministic) equations that assign values to some random variables, (iii) a probability distribution that assigns values to the rest of the variables. The variables whose values are assigned by the probability distribution are called exogenous.

SEM's are conveniently represented as directed acyclic graphs (DAG). The nodes represent random variables, and the edges represent {\it direct} causal relationships between variables: there is an edge from node $i$ to node $j$ if the equation that assigns value to variable $j$ takes variable $i$ as input. The nodes that have no parents represent the exogenous variables.

To sample from an SEM, we start by assigning values to the root nodes by sampling from the probability distribution and recursively assign values to the other nodes by the equations. Thus the nodes whose values are assigned by equations are random variables on the probability space ``generated by'' the exogenous variables.
In this setting, counterfactuals are defined as random variables whose values are assigned by a modified SEM, where the equations and/or the probability distribution are modified according to the premise of the counterfactual. We wrap up our brief overview of counterfactuals with an example and refer to \cite{Pearl2016}, Chapter 4 for further details.

\begin{example}
Consider the intervention $\ba\gets a$ and the counterfactual $\by_{\ba\gets a}$ in the SEM depicted in Figure \ref{fig:pre-intervention}. The counterfactual is the counterpart of $\by$ in the modified SEM depicted on the right of Figure \ref{fig:post-intervention-1}, in which the equation that assigns the value of $\ba$ is replaced by the equation $\ba = a$. We see that the value of $\by_{\ba\gets a}$ ultimately depends on the values of the exogeneous variables in the SEM, making it a random variable on the same probability space as $\by$.
Thus it is possible to evaluate ``cross-SEM'' probabilities such as $\cL(\by_{\ba\gets a}\mid\by = y)$.
We remark that this SEM formalism allows us to study the effects of more sophisticated interventions such as $\ba\sim\bbP_a$ (\cf\ Figure \ref{fig:post-intervention-2}).

\begin{figure}
  \centering
  \begin{subfigure}[b]{0.32\textwidth}
    \tikz{
      \node[latent]                            (u) {$\bu$};
      \node[latent, below=of u, xshift=-1.2cm] (a) {$\ba$};
      \node[latent, below=of u, xshift=1.2cm]  (z) {$\bz$};
      \node[latent, below=of a, xshift=1.2cm]  (y) {$\by$};
      \edge {u} {a,z}
      \edge {a} {z}
      \edge {z} {y}
    }
    \caption{}
    \label{fig:pre-intervention}
  \end{subfigure}
  \begin{subfigure}[b]{0.32\textwidth}
    \tikz{
      \node[latent]                            (u) {$\bu$};
      \node[obs, below=of u, xshift=-1.2cm]    (a) {$a$};
      \node[latent, below=of u, xshift=1.2cm]  (z) {$\bz_{\ba\gets a}$};
      \node[latent, below=of a, xshift=1.2cm]  (y) {$\by_{\ba\gets a}$};
      \edge {u} {z}
      \edge {a} {z}
      \edge {z} {y}
    }
    \caption{}
    \label{fig:post-intervention-1}
  \end{subfigure}
  \begin{subfigure}[b]{0.32\textwidth}
    \tikz{
      \node[latent]                            (u) {$\bu$};
      \node[latent, below=of u, xshift=-1.2cm] (a) {$\ba_{\ba\sim\bbP_a}$};
      \node[latent, below=of u, xshift=1.2cm]  (z) {$\bz_{\ba\sim\bbP_a}$};
      \node[latent, below=of a, xshift=1.2cm]  (y) {$\by_{\ba\sim\bbP_a}$};
      \edge {u} {z}
      \edge {a} {z}
      \edge {z} {y}
    }
    \caption{}
    \label{fig:post-intervention-2}
  \end{subfigure}
  \label{fig:intervention}
\end{figure}
\end{example}

In the rest of this subsection, we describe two counterfactual notions of non-discrimination. The first was proposed recently by \cite{Kusner2017Counterfactual}, while the second is suggested by CP. To keep things simple, we specialize to supervised learning and focus on prediction.

\begin{definition}[counterfactual fairness (CF) \cite{Kusner2017Counterfactual}]
\label{def:cf}
A prediction $\hby$ is counterfactually fair with respect to sensitive attribute $\ba$ in light of evidence $\{\be = e\}$ iff
\begin{equation}
\cL(\hby_{\ba \gets a} \mid \be=e) = \cL(\hby_{\ba \gets a'} \mid \be=e)\text{ for any }a,a'\in\cA.
\label{eq:cf}
\end{equation}
If \eqref{eq:cf} holds for all $e \in \cE$, $\hby$ is counterfactually fair with respect to sensitive attribute $\ba$ in light of evidence $\be$.
\end{definition}

In Definition \ref{def:cf}, $\be$ is the evidence we observe in the real world. Although it plays the part of $\bz$ in CP, we call it evidence and denote it by $\be$ to emphasize it is observed. To see that CF is an instance of CP, let $\hby_\ba$, $\ba_\ba$ be the counterparts of $\hby$, $\ba$ in a modified SEM, where the step that assigns value to $\ba$ is replaced by $\ba\sim\unif(\cA)$, and note that \eqref{eq:cf} is equivalent to
\begin{equation}
\cL(\hby_\ba \mid \be=e, \ba_\ba = a) = \cL(\hby_\ba \mid \be=e, \ba_\ba = a')\text{ for any }a,a'\in\cA,
\label{eq:cf2}
\end{equation}
We remark that the law of the intervention is unimportant because we condition on the value of $\ba$. We pick $\ba\sim\unif(\cA)$ to keep things concrete.

The notion of CF is best illustrated by the following case on employment discrimination. In \cite{1996Carson}, the judges wrote ``the central question in any employment-discrimination case is whether the employer would have taken the same action had the employee been of a different race (age, sex, religion, national origin, etc.) and everything else had been the same''. In other words, to ascertain whether discrimination occurred, the judges compared the employee with his counterpart in a counterfactual world, rather than a similar employee in the real world.

We remark that it may not be possible to follow the judges directive literally and keep all other attributes the same: the intervention $\ba\gets a$ may propagate in the modified SEM and lead to discrepancies with the evidence. For example, consider a female employee who is homosexual. In a counterfactual world where she is male, it is not possible to keep both her sexual orientation and the gender she is attracted the same as hers in the real world.

As we saw, CF is an instance of CP where we segment the population by observable evidence. A related notion of non-discrimination is equalized counterfactual odds (ECO): it is an instance of CP that segments the population by counterfactual attributes. It is motivated by Example \ref{ex:driving-risk}.

\begin{example}
\label{ex:driving-risk}
Consider a system that predicts a driver's accident risk from his or her driving record. The prediction $\hby$ depends directly on a driver's driving record $\bz$, which in turn depends on the driver's driving ability $\bu$.
Driving ability also directly affects a driver's accident risk $\by$ and whether he or she is disabled $\ba$ (poor drivers tend to get into accidents,  which cause disabilities). Figure \ref{figeco} is a DAG that depicts the functional dependencies among the variables.

\begin{figure}
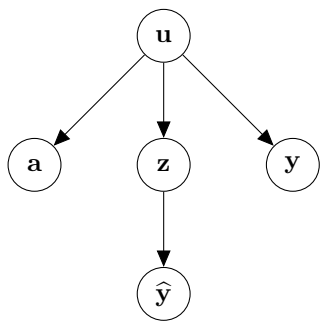

  \tikz{
    \node[latent] (u) {$\bu$};
    \node[latent, below=of u] (z) {$\bz$};
    \node[latent, left=of z] (a) {$\ba$};
    \node[latent, right=of z] (y) {$\by$};
    \node[latent, below=of z] (hy) {$\hby$};
    \edge {u} {z,a,y};
    \edge {z} {hy}
    }
  \caption{SEM of accident risk prediction (Example \ref{ex:driving-risk})}
  \label{figeco}
\end{figure}

This example does not satisfy EO: the path $\ba\gets\bu\to\bz\to\hby$ between $\ba$ and $\hby$ is not blocked. However, $\hby$ is intuitively non-discriminatory:
there is only dependence between $\ba$ and $\hby$ because driving ability is a parent of disability and driving record. the prediction $\hby$ has no causal dependence on $\ba$ does not penalize good drivers that happen to be disabled.

We see that EO is too stringent a condition in this scenario. It not only prohibits the prediction from treating disabled drivers differently because of their disability, but also prohibits the prediction from happening to put disabled drivers in an disadvantaged position due to the presence of a confounder. In this scenario, disabled people happen to have worse driving records because driving ability affects one's driving record and causes disability.
\end{example}

Example \ref{ex:driving-risk} shows that EO is too stringent because it prohibits probabilistic dependence between $\hby$ and $\ba$, which may arise due to confounding. The notion of equalized counterfactual odds is an amendment of EO that only prohibits causal relationships between $\hby$ and $\ba$.

\begin{definition}[equalized counterfactual odds (ECO)]
A prediction $\hby$ of $\by\in\cY$ satisfies equalized counterfactual odds with respect to protected attribute $\ba$ conditioned on $\by_\ba=y$ iff
\begin{equation}
\cL(\hby_\ba \mid \by_\ba=y, \ba_\ba = a) = \cL(\hby_\ba \mid \by_\ba=y, \ba_\ba = a')\text{ for all }a,a'\in\cA.
\label{def:eco}
\end{equation}
If \eqref{def:eco} holds for all $y \in \cY$, we say $\hby$ satisfies equalized counterfactual odds with respect to $\ba$.
\end{definition}

In a nutshell, ECO is EO on a modified SEM, in which the step that assigns value to $\ba$ is replaced by $\ba\sim\bbP_a$. The graph of the modified SEM is identical to that of the original SEM, except all the edges that point to $\ba$ are removed. This removes all back door paths between $\hby$ and $\ba$, which typically represent the effects of confounders. Thus ECO only prohibits probabilistic dependence between $\hby$ and $\ba$ in the original SEM through front door paths. This leads to a simple way of verifying ECO.

\begin{lemma}
A prediction $\hby$ of $\by$ satisfies ECO with respect to protected attribute $\ba$ if any front door paths from $\ba$ to $\hby$ are blocked by $\by$.
\end{lemma}

ECO is also closely related to CF: both compare the law of the counterfactual prediction $\hby_\ba$ on segments of the population. ECO segments the population by the counterfactual target $\by_{\ba}$, while CF segments the population by (observable) evidence $\be$. In practice, CF is a fairly stringent notion of non-discrimination. As \cite{Kusner2017Counterfactual} point out, there are instances in which perfect prediction does not satisfy CF. On the other hand, perfect prediction always satisfies ECO. To wrap up, we present another example that highlights the difference between the two notions.

\begin{example}
\label{ex:priest-hiring}
Consider an SEM of a church's priest hiring process. The church's hiring decision $\hby = \ones\{\bz \ge 1.8\}$ depends on an applicants score $\bz=\ba \bu$, where $\bu \sim \text{unif}(0,2)$ is the applicant's propensity for priest work and $\ba\in\{0,1\}$ is whether the applicant is Christian. Figure \ref{fig:ceo} depicts an SEM of the priest hiring process.

\begin{figure}
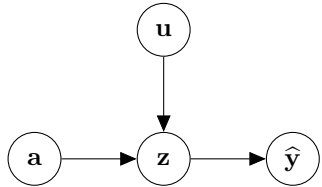

  \tikz{
    \node[latent] (u) {$\bu$};
    \node[latent, below=of u] (z) {$\bz$};
    \node[latent, left=of z] (a) {$\ba$};
    \node[latent, right=of z] (hy) {$\hby$};
    \edge {u,a} {z}
    \edge {z} {hy}
  }
  \caption{SEM of priest hiring process (Example \ref{ex:priest-hiring})}
  \label{fig:ceo}
\end{figure}

Consider an atheist applicant whose talents are well-suited to priest work (\eg\ charismatic, persuasive, $\bu = 1.9$). He applied for the position, but was rejected. Since he would have been hired if he was a Christian, the hiring process is not counterfactually fair with respect to $\ba$ in light of evidence $\bz$. Graphically, conditioning on $\bz$ in the unmodified SEM does not block the path between $\ba_a$ and $\hby_a$ in the graph of the modified SEM. On the other hand, the hiring process clearly satisfies ECO: $\bz$ blocks the only front door path between $\ba$ and $\hby$ in the graph of the (unmodified) SEM.
\end{example}

Before moving on, we mention a few recently proposed counterfactual notions of non-discrimination. \cite{Zhang2016causal} and \cite{Nabi2017Fair} formalize discrimination as the presence of path specific effects (\cf\ \cite{Pearl2009Causality}, \S 4.5.3). Although path-specific notions of non-discrimination are also instances of CP, we skip the details here. \cite{Kilbertus2017Avoiding} addresses the difficulty of modeling and determining the effect of intervening on protected attributes by considering non-discrimination with respect to proxies of protected attributes.


To wrap up, we cite a few related works on non-discrimination in machine learning. DP as a notion of non-discrimination was studied in \cite{Zemel2013Learning}. \cite{Friedler2016im} extends the notion of individual fairness to distinguish between {\it constructs}, which are unobservable attributes (\eg\ intelligence), and {\it observations} (\eg\ score on IQ test), which are proxies of constructs that enter into the algorithm. \cite{Berk2017Fairness} reviews various notions of non-discrimination in the criminal justice system.

\section{Conditional parity by randomization}
\label{sec:randomization}

In supervised learning, we observe realizations of $(\ba,\bs,\by)$, where $\ba\in\{0,1\}$ is the protected attribute, $\bs\in\reals^d$ is a score, $\by$ is the outcome. In general, $\bs$ is dependent on $\ba$. If we wish to obtain a prediction that does not depend on the protected attribute we have to sacrifice some efficiency and use a randomized procedure. In this section we consider the construction of a (randomized) decision rule $\hby=\hby(\bs,\ba)$ such that ${\cal L}(\hby\mid \ba=a,\by=y)$ does not depend on $a$.

Assume for simplicity that $\bs$ is discrete. In that case let $f_{ya}$ be the probability vector corresponding to ${\cal L}(\bs\mid \ba=a,\by=y)$. A non-discriminatory randomization is a pair of Markov kernels, $K_0,K_1\in \mathbb{R}^{k\times k_1}$ satisfying
\begin{equation}
\label{eq:const}
  \begin{split}
     f_{y1}K_1 &= f_{y0}K_0, \qquad y\in\{0,1\}
     \\
      K_0(i,j), K_1(i,j) &\ge 0,\qquad i=1,\dots,k,\; j=1,\dots, k_1
      \\
      \sum_{j=1}^{k_1}K_m(i,j)&= 1,\qquad i=1,\dots,k, m=0,1.
  \end{split}
\end{equation}
The minor difficulty is due to the fact that the conditional density given $\by$ should be checked, while the randomization cannot depend on $\by$ which is unobserved at the time of the randomization.
\begin{lemma}
  In general, we need to randomize the score of both categories to achieve EO.
\end{lemma}
\begin{proof}
  The lemma follows from the fact that a Markov kernel is a contraction in two measures, and in general they are unrelated.  Consider  two sets of densities such that $\|f_{10}-f_{00}\|_1\ll\|f_{11}-f_{01}\|_1$, but on a small interval, small enough such that it has a little contribution to the $L_1$ distance,  $f_{10}(x)/f_{00}(x)\gg f_{11}(x)/f_{01}(x)$.   If it was that  $f_{y0}K=f_{y1}$ 
  \begin{equation*}
    \begin{split}
       \|f_{11}-f_{01}\|_1 & = \sum_j \Bigl| \sum_i \bigl(f_{10}(i)-f_{00}(i)\bigr)K_1(i,j)\Bigr|
       \\
         &\le \sum_i\big|f_{11}(i)-f_{01}(i)\bigr|\sum_j K(i,j) = \|f_{11}-f_{01}\|_1.
\end{split}
\end{equation*}
which contradict the assumption. On the other hand, if it was that  $f_{y1}K=f_{y0}$
\begin{equation*}
    \begin{split}
         \max_j \frac{f_{10}(j)}{f_{00}(j)} &= \max_j\frac{\sum_i f_{11}(i)K_1(i,j)}{\sum_i f_{01}(i)K_1(i,j)}
         \\
         &=\max_j\frac{\sum_i \frac {f_{11}(i)}{f_{01}(i)}f_{01}(i)K_1(i,j)}{\sum_i f_{01}(i)K_1(i,j)}
         \\
         &\le \max_i \frac{f_{11}(i)}{f_{01}(i)},
    \end{split}
  \end{equation*}
 which contradict the other assumption. Hence neither it is that  $f_{y0}K=f_{y1}$ nor that  $f_{y0}K=f_{y1}$. 
  
  \end{proof}

 The set  \eqref{eq:const} has $2(k+k_1-1)$ equality constraints with $2kk_1$ undefined parameters. It always has the trivial solution.  Adding a linear cost function, \eg\
 \begin{equation}\label{eq:target}
   \sum_m\sum_i\sum_{j}c_{i}|j-j(i)|^\alpha K_m(i,j),
 \end{equation}
 for example $j(i)=\lceil ik_1/k\rceil$, turns the feasibility problem into a linear program. To avoid sparse solutions, we add another set of constraints
  \begin{equation}\label{eq:const2}
  \sum_j jK_m(i,j)\text{ is monotone non-decreasing in } i \text{ for }m=1,2.
  \end{equation}
That is, the rows of $K_m$ are increasing in the mean.

\begin{example}
\label{ex:sat}Consider the following model in which there is a never observed latent variable $\bz$
  \begin{equation*}
  \begin{split}
    \ba&\in\{0,1\}
    \\
    \bz\mid\ba=a &\sim N(\mu_za,\tau_z^2)
    \\
    \bs\mid\ba=a,\bz=z &\sim N(z+\mu_sa,\sigma^2_s)
    \\
    \by&\in \{0,1\}
    \\
    P(\by=1\mid \ba,\bs,\bz)&= p_z(\bz).
    \end{split}
  \end{equation*}
In words $\bz$ is the ability of the random subject, the two groups are of different ability. Let assume that $\mu_z,\mu_s\ge 0$ We have a noisy observations $\bs$ which is biased in favor of the stronger group. Since
\begin{equation*}
  \begin{split}
     \bz \mid \ba=a,\bs=s & \sim N(\frac{\sigma^2_s}{\sigma_s^2+\tau^2_z}a\mu_z + \frac{\tau^2_z}{\sigma_s^2+\tau^2_z}(s-a\mu_s), \frac{\sigma^2_s\tau^2_z}{\sigma_s^2+\tau^2_z})  ,
  \end{split} \end{equation*}
if we have $\mu_s=(\sigma^2_s/\tau_z^2)\mu_z$ then the Bayes estimator of $\bz$ given $\ba,\bs$ is the same for the two groups. In reality, biasing the score of the stronger group is not done explicitly. However, in order to improve the Bayes estimator, a culturally dependent criteria may be introduced, which in effect biases the score. This situation may seem fair. The score (e.g., the SAT) is used in the same way independently of the attribute and it is the optimal way to use the score as it is Bayes. Removal of the bias would result in an inferior selection procedure.

However, this is a discriminatory policy by other criteria.
In Figure \ref{fig:bayesprice} we present the situation. The score distributions of the two groups are different, but worse, between two subjects with the same latent ``ability'', $\bz$, the subject from the stronger group gets, on the average, a higher score $\bs$.

\begin{figure}
  \begin{subfigure}[b]{0.48\textwidth}
    \includegraphics[width=\textwidth]{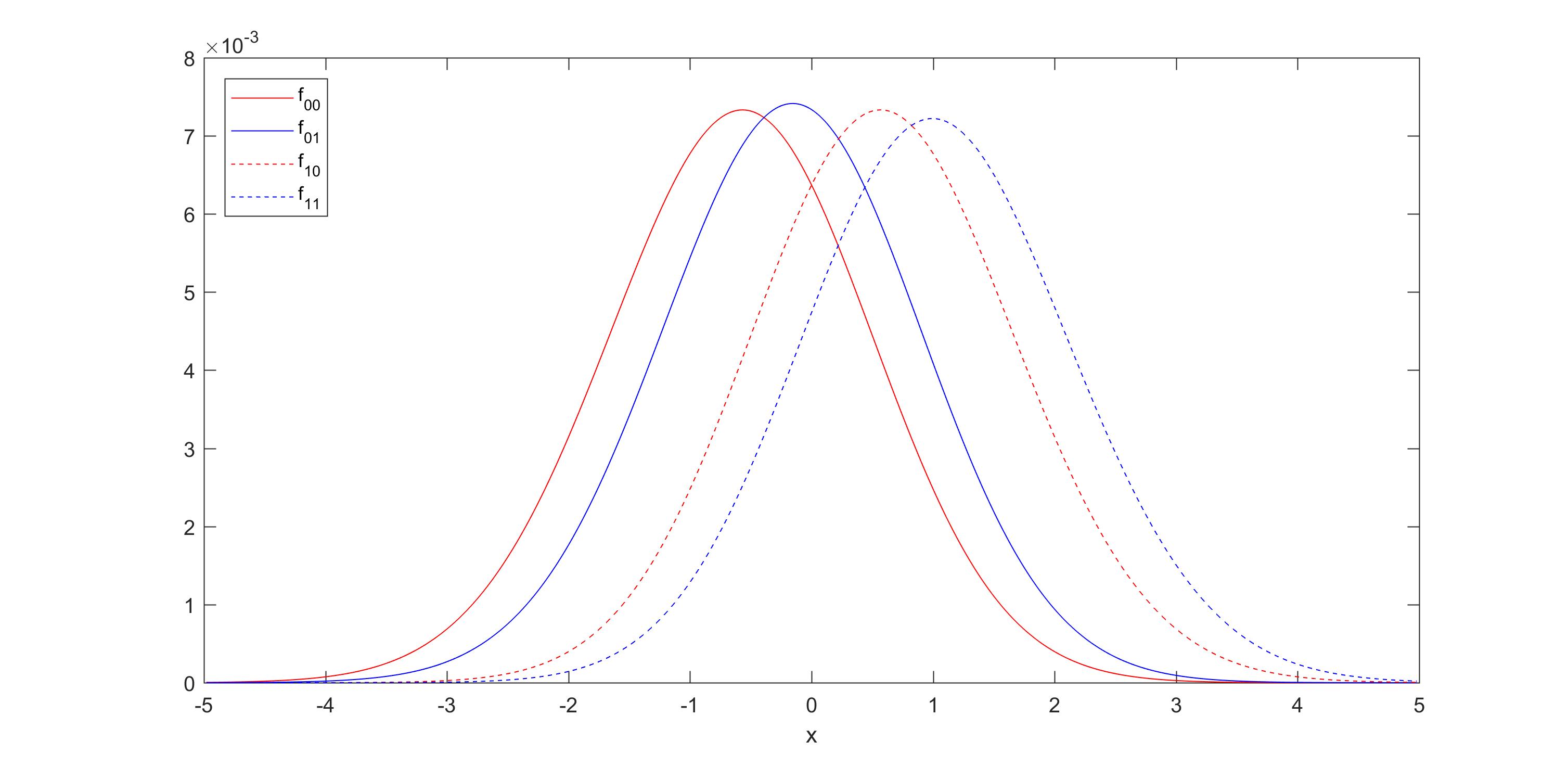}
    \caption{}
  \end{subfigure}
  ~ 
  \begin{subfigure}[b]{0.48\textwidth}
    \includegraphics[width=\textwidth]{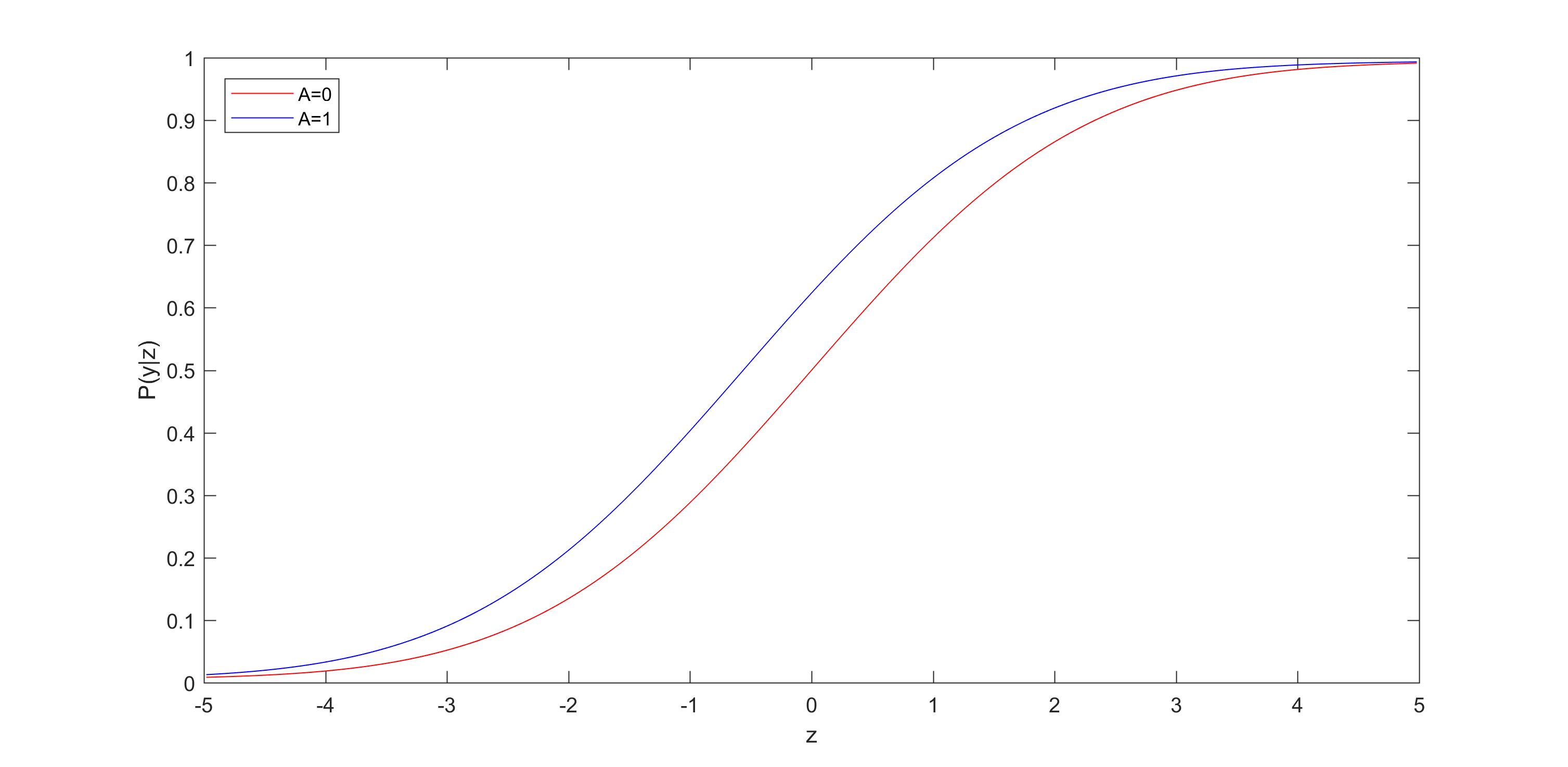}
    \caption{}
  \end{subfigure}
  \caption{(a) The conditional densities given $\by$ and $\ba$. (b) $P(\hat\by=1|\bz)$ vs. $\bz$. Among two candidates with the same latent ability, the one who comes from the weaker group gets on the average  a smaller Bayes decision.}
  \label{fig:bayesprice}
\end{figure}

The minimizer of \eqref{eq:target} subject to the \eqref{eq:const} and \eqref{eq:const2} constraints is presented in Figure \ref{fig:markov}.
\begin{figure}
  \includegraphics[width=0.75\textwidth]{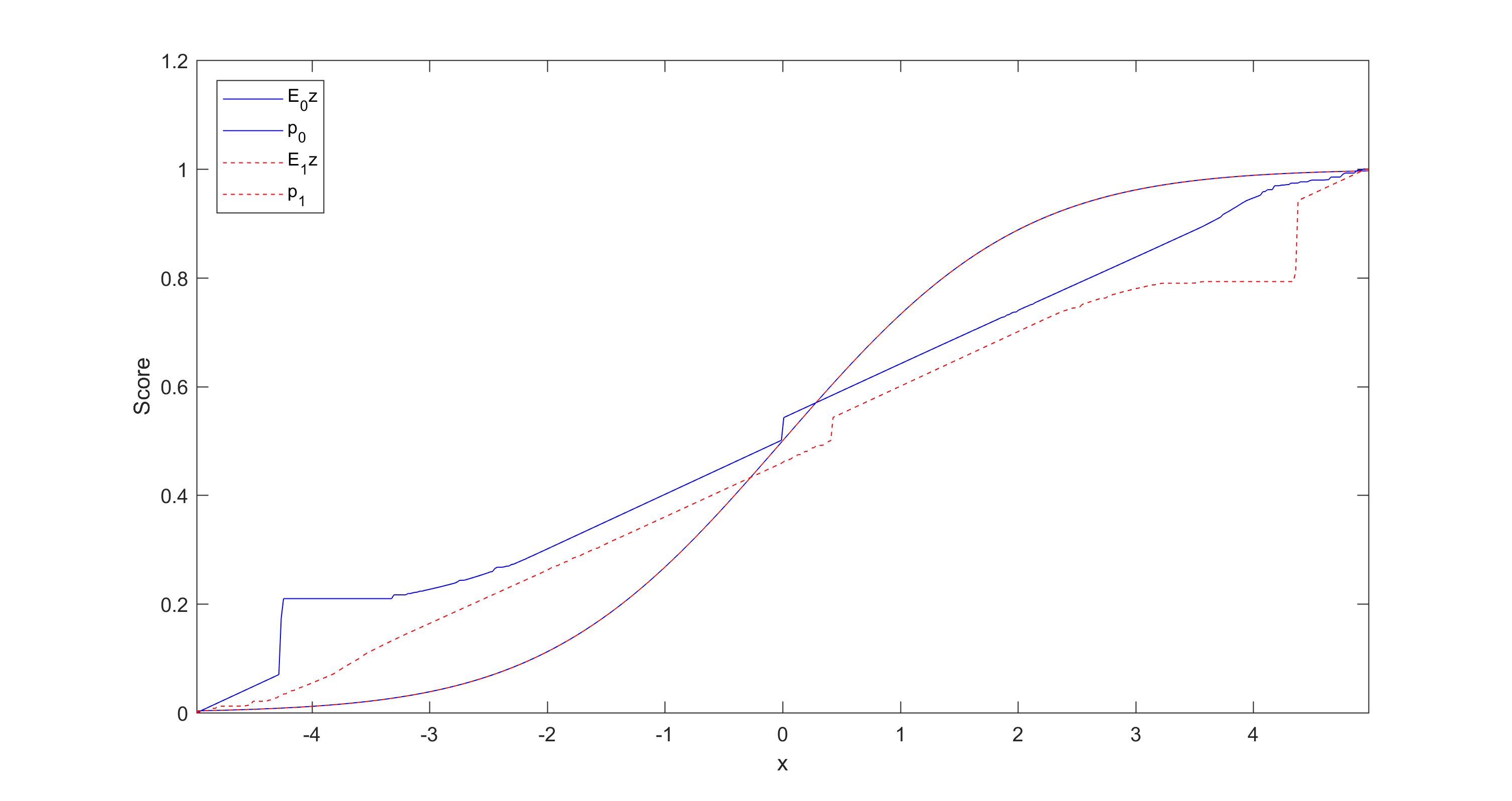}
  \caption{The Bayes estimator and the solution of the linear program non-discrimantory algorithm. }
  \label{fig:markov}
\end{figure}
The Breir score is presented in Table \ref{tab:breir}.
\begin{table}
  \centering
  \caption{The (minimal) impact of the type of decision on the Breir score}\label{tab:breir}
  \vspace{6pt}
  \begin{tabular}{|l|r|r|}
    \hline
    Decision & $\ba=0$ & $\ba=1$ \\\hline
    Bayes decision & 0.1898 & 0.1820 \\
    Bayes decision based on 2 categories & 0.2064 & 0.1990 \\
    The non-discrimantory decision & 0.2092 & 0.2001 \\
    \hline
  \end{tabular}
\end{table}

Finally, in Figure \ref{fig:contMarkov} we present the output distribution of the Markov kernel when $k_1=500$.

\begin{figure}
  \begin{subfigure}[b]{0.48\textwidth}
    \includegraphics[width=\textwidth]{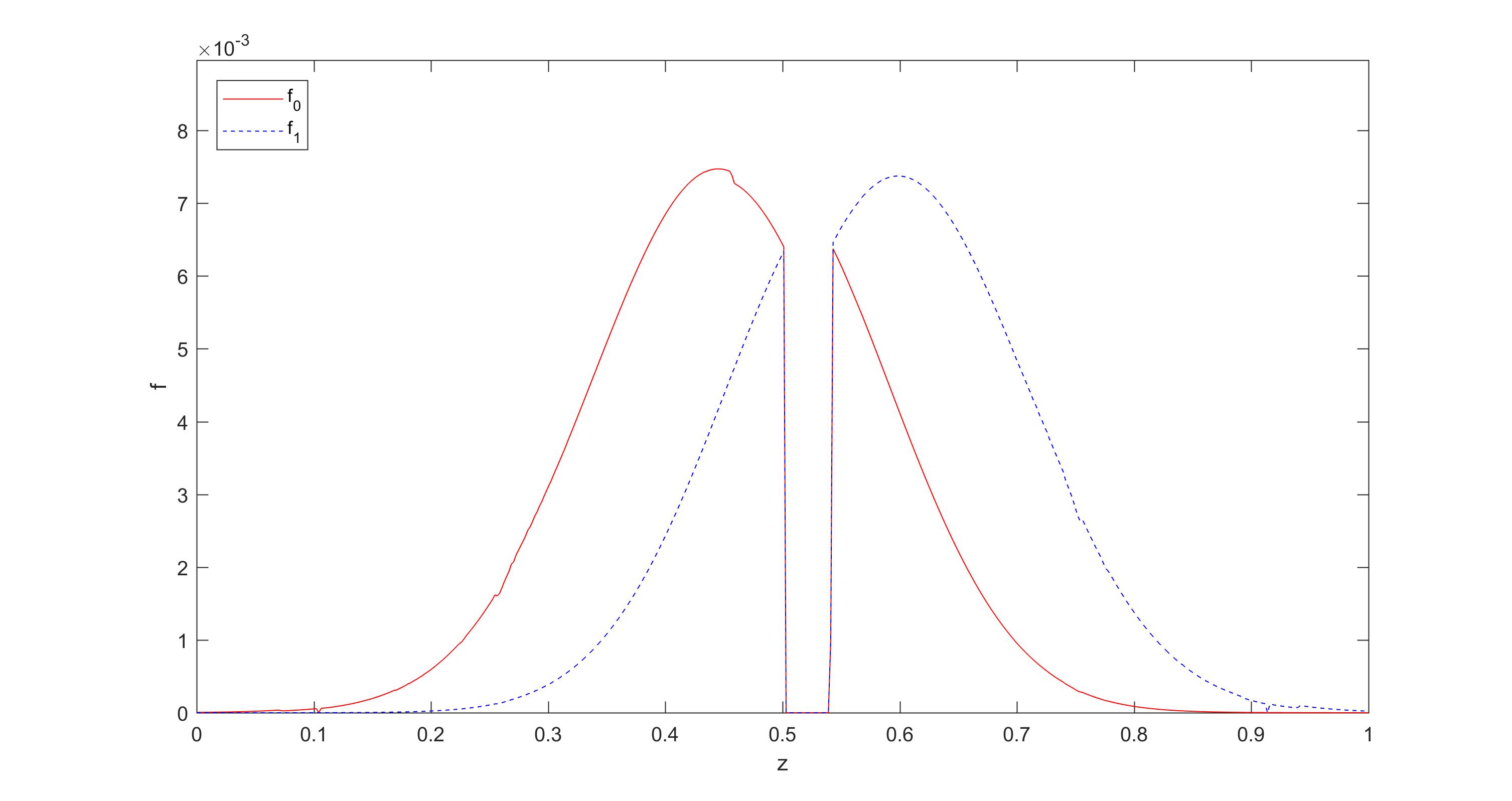}
    \caption{}
  \end{subfigure}
  ~ 
  \begin{subfigure}[b]{0.48\textwidth}
    \includegraphics[width=\textwidth]{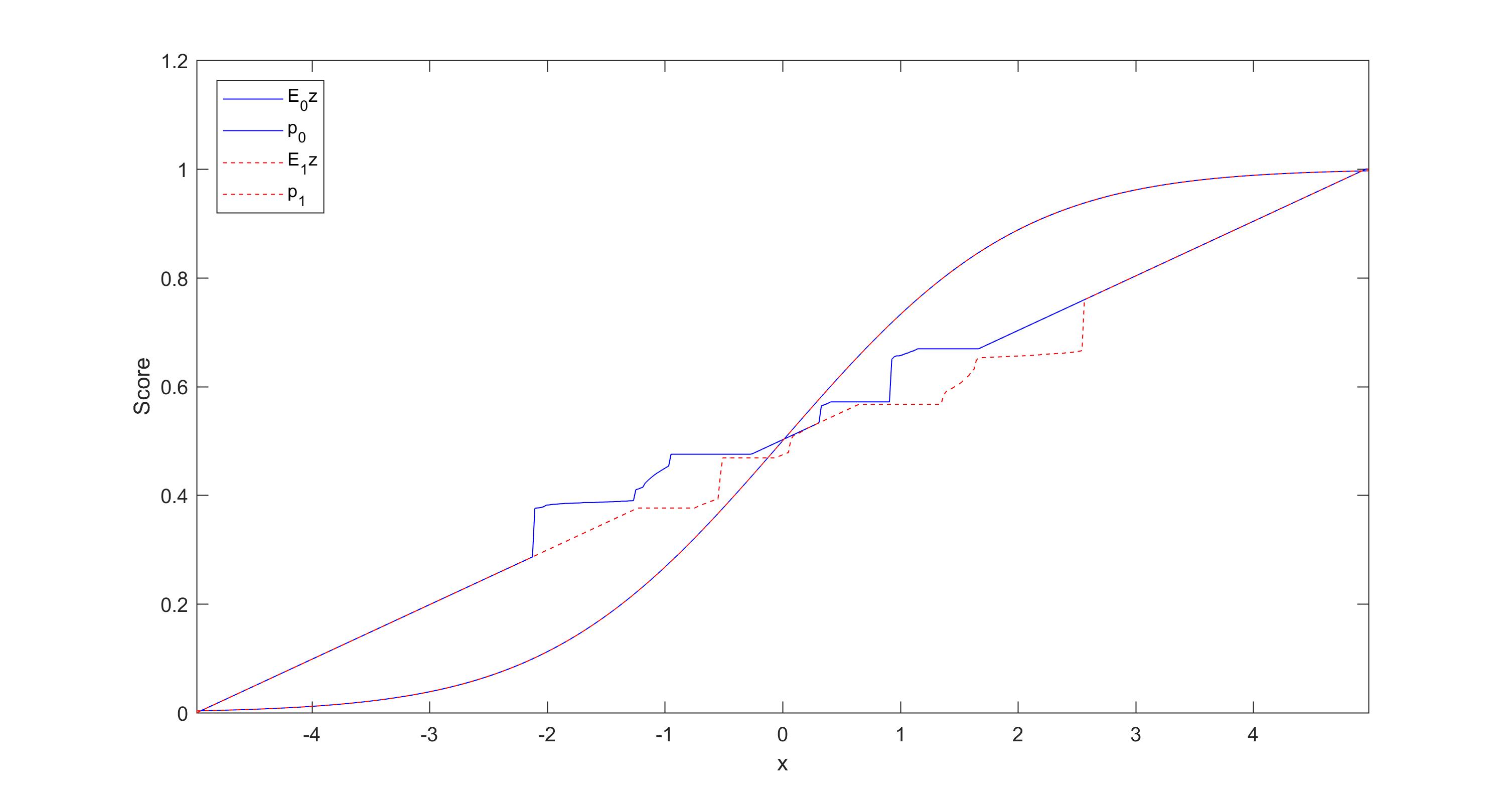}
    \caption{}
  \end{subfigure}
  \caption{The output  of the linear program. (a) The conditional density given $\by$; (b) The expected value of the output given $\bs$ conditional on $\ba$ and the Bayes estimator, which is independent of $\ba$. }
  \label{fig:contMarkov}
\end{figure}
\end{example}

The situation is more complicated when the outcome is more than binary valued, or continuous. However, it is simple enough if the outcome and raw prediction are jointly normal conditioned on the protected attribute. Consider, for example, the situation of Example \ref{ex:sat}. The concluding model is that $\bs\mid \by,\ba \sim \cN(\mu_\ba+A_\ba \by, \Sigma_\ba)$.
\begin{lemma}
  Suppose $\bs\in R^{n_s}$,  $\by\in R^{n_y}$ and $A_a\in   R^{n_s\times n_y}$, $a\in\{0,1\}$, are full rank. Then, without loss of generality, we can assume $\mu_a=0$, $a=\{0,1\}$, and $A_0=A_1$. A minimal randomized procedure that equates the conditional distribution of $\bs$ is given by $\hby|\ba,\bs = \cN(\bs, T_a)$, where $T_a=\sum_{(a/2-1)\lambda_i<0}\lambda_i\xi_i\xi_i'$, and $(\lambda_1,\xi_1),\dots,(\lambda_{n_s},\xi_{n_s})$ is an orthonormal eigen-system of $\Sigma_1-\Sigma_0$.
\end{lemma}
\begin{proof}
  We can always transform $\bs$ when $\ba=0$ by the transformation
  \begin{equation*}
    \bs\to \begin{cases}
      A_1A_0^\dagger(\bs-\mu_0), &\mbox{if } a=0 \\
      \bs-\mu_1&\mbox{if } a=1,
    \end{cases}
    \end{equation*}
    which equate the conditional mean of $\bs$ under $\ba\in\{0,1\}$.

    Now, since $T_0-T_1=\Sigma_1-\Sigma_0$, ${\cal L}(\hby\mid\by,\ba=a)$ does not depend on $a$. Finally, since $A_1$ has a full rank, $E\bigl(K_1(\bs,\cdot) - K_0(\bs,\cdot)\mid \by)=0 $, iff $K_1-K_0=0$ (the shift is a complete sufficient statistics in the multivariate normal distribution. This is achieved by the randomization given above.
\end{proof}

\section{Testing conditional parity}

In this section, we describe a kernel-based approach to testing CP developed by \cite{Zhang2012}. We begin by characterizing the conditional independence condition $\bx\perp\ba\mid\bz$ in terms of cross-covariance operators. Let $k_x$, $k_a$, and $k_z$ be positive definite kernels on $\cX$, $\cA$, and $\cZ$ respectively and $\cH_x$, $\cH_a$, and $\cH_z$ be the respective reproducing kernel Hilbert spaces (RKHS). Throughout this section, we assume the kernels satisfy
\[
\Ex\bigl[k_x(\bx,\bx)\bigr] < \infty,\quad\Ex\bigl[k_a(\ba,\ba)\bigr] < \infty,\quad\Ex\bigl[k_z(\bz,\bz)\bigr] < \infty.
\]

For any probability distribution on $\cX$, its {\it RKHS embedding} is the unique $\mu_x\in\cH_x$ such that
\[
\langle\mu_x,f\rangle_x = \Ex\bigl[f(\bx)\bigr]\text{ for any }f\in\cH_x
\]
for any $f\in\cH_x$. It is well-defined because the assumptions on $k_x$ imply $f\to\Ex\bigl[f(\bx)\bigr]$ is a bounded linear functional. By the reproducing property, we see that $\mu_x$ has the explicit form
\[
\mu_x(x) = \langle k_x(x,\cdot),\mu_x\rangle_x = \Ex\bigl[k_x(\bx,x)\bigr].
\]

The {\it cross-covariance operator} of $(\bx,\ba)$ is an operator from $\cH_a$ to $\cH_x$ such that
\[
\langle f,\Sigma_{x,a} g\rangle_x = \cov\bigl[f(\bx),g(\ba)\bigr]\text{ for any }f\in\cH_x\text{ and }g\in\cH_a.
\]
It is the functional analogue of the covariance matrix a pair of random vectors. In terms of the kernel and the RKHS embeddings of the marginal distributions of $\bx$ and $\ba$, it has the form
\[
\langle f,\Sigma_{x,a}g\rangle_x = \Ex\bigl[\langle f,k_x(\bx,\cdot) - \mu_x\rangle_x\langle k_a(\ba,\cdot) - \mu_a,g\rangle_a\bigr].
\]
Letting $f = k(x,\cdot)$, we see that $\Sigma_{x,a}g$ is
\[
(\Sigma_{x,a}g)(x) = \langle k(x,\cdot),\Sigma_{x,a}g\rangle_x = \Ex\bigl[(k_x(\bx,x) - \mu_x(x))\bigl(g(\ba) - \Ex\bigl[g(\ba)\bigr]\bigr)\bigr].
\]
The cross-covariance operator of $(\bx,\bx)$ is a positive self-adjoint operator and is called the {\it covariance operator} of $\bx$. The {\it conditional cross-covariance operator} of $(\bx,\ba)$ given $\bz$ is
\[
\Sigma_{x,a\mid z} = \Sigma_{x,a} - \Sigma_{x,z}\Sigma_{z,z}^{-1}\Sigma_{z,a}.
\]
Under some technical conditions, \cite{Fukumizu2008Kernel} show that it is an operator from $\cH_a$ to $\cH_x$ such that
\[
\langle f,\Sigma_{x,a\mid z}g\rangle_x = \Ex\bigl[\cov\bigl[f(\bx),g(\ba)\mid\bz\bigr]\bigr].
\]

Before we state the functional characterization of conditional independence, we define some additional notation. The tensor product $\cH_x\otimes\cH_a$ is an RKHS equipped with the inner product
\[
\langle f\otimes g,x\otimes y\rangle = \langle f,x\rangle_x\langle g,y\rangle_a.
\]
We extend this inner product to all of $\cH_x\otimes\cH_a$ by bilinearity. We see that the representer of evaluation in $\cH_x\otimes\cH_a$ is the outer product of the representers of evaluation in $\cH_x$ and $\cH_a$:
\[
\langle k(x,\cdot)\otimes k(y,\cdot),f\otimes g\rangle = f(x)g(y),
\]
and the kernel is the pointwise product of $k_x$ and $k_y$:
\[
\begin{aligned}
(k_x\cdot k_y)((x_1,y_1),(x_2,y_2)) &= \langle k(x_1,\cdot)\otimes k(y_1,\cdot),k(x_2,\cdot)\otimes k(y_2,\cdot)\rangle \\
&= k_x(x_1,x_2)k_y(y_1,y_2).
\end{aligned}
\]
We are ready to state the functional characterization of conditional independence by \cite{Fukumizu2008Kernel}.

\begin{theorem}[\cite{Fukumizu2004,Fukumizu2008Kernel}]
\label{eq:rkhs-characterization-con-ind}
Let $k_{x,z} := k_x\cdot k_z$ be a kernel on $\cX\times\cZ$ and $\cH_x\otimes\cH_z$ be its RKHS. As long as $k_{x,z}\cdot k_a$ is a characteristic kernel \footnote{We call kernel is characteristic if the RKHS embedding $\bbP\to\Ex_{\bbP}\bigl[k(\bx,\cdot)\bigr]$ is injective. In other words, $\Ex_{\bbP}\bigl[f(\bx)\bigr] = \Ex_{\bbQ}\bigl[f(\bx)\bigr]$ for all $f\in\cH$ implies $\bbP=\bbQ$.} on $\cX\times\cZ\times\cA$ and $\cH_z\oplus\reals$ is dense in $L^2(\bbP_z)$, where $\oplus$ denotes direct sum and $\reals$ is the space of constant functions, we have
\[
\Sigma_{(x,z),a\mid z} = 0\iff\bx\perp\ba\mid\bz.
\]
\end{theorem}

The non-trivial implication in Theorem \ref{eq:rkhs-characterization-con-ind} is the ``only if'' implication. If $\Sigma_{(x,z),a\mid z} = 0$, we have
\[
0 = \Ex\bigl[\cov\bigl[f(\bx)h(\bz),g(\ba)\mid\bz\bigr]\bigr] = \Ex\bigl[\cov\bigl[f(\bx),g(\ba)\mid\bz\bigr]h(\bz)\bigr]
\]
for any $h\in\cH_z$, which implies $\cov\bigl[f(\bx),g(\ba)\mid\bz\bigr] = 0$ as long as $\cH_z$ is rich enough. The assumption $\cH_z + \reals$ is dense in $L^2(\bbP_z)$ ensures $\cH_z$ is rich enough.

In light of Theorem \ref{eq:rkhs-characterization-con-ind}, a natural test statistic is the Hilbert-Schmidt (HS) norm of a plug in estimator of the conditional cross-covariance operator. Let the {\it empirical cross-covariance operator} of $(\bx,\ba)$ be
\[
\hSigma_{x,a} := \frac1n\sum_{i=1}^nk_x(\bx_i,\cdot)\otimes k_a(\ba_i,\cdot) -\hmu_x\otimes\hmu_a,
\]
where $\hmu_x := \frac1n\sum_{i=1}^nk_x(\bx_i,\cdot)$ (resp.\ $\hmu_a$). The {\it empirical conditional cross-covari\-ance operator} is
\[
\hSigma_{(x,z),a\mid z} := \hSigma_{(x,z),a} - \hSigma_{(x,z),z}(\hSigma_{z,z} + \lambda I_z)^{-1}\hSigma_{z,a},
\]
where $\lambda > 0$ is a regularization parameter. It is possible to express its HS norm in terms of the kernel matrices $\bG_x$, $\bG_a$, $\bG_z$, where $\bigl[\bG_x\bigr]_{i,j} = k_x(\bx_i,\bx_j)$ (resp.\ $\bG_a$, $\bG_z$).

\begin{lemma}
\label{lem:hsic-test-stats}
We have
\[
\begin{aligned}
\|\hSigma_{(x,z),a\mid z}\|_{\HS}^2 &= \frac{1}{n^2}\tr(\bK_{x,z}\bK_a) - \frac{2}{n^2}\tr(\bK_{x,z}\bK_z(\bK_z + \lambda M_n)^\dagger\bK_y) \\
&\quad+ \frac{1}{n^2}\tr((\bK_z + \lambda M_n)^\dagger\bK_z\bK_{x,z}\bK_z(\bK_z + \lambda M_n)^\dagger\bK_y),
\end{aligned}
\]
where $\bK_{x,z} = \bK_x\cdot\bK_z$.
\end{lemma}

\cite{Zhang2012} show that the test statistic is asymptotically a mixture of independent $\chi_1^2$ random variables
\[\textstyle
n\|\hSigma_{(x,z),a\mid z}\|_{\HS}^2 \overset{d}{\to} \sum_{i=1}^\infty\lambda_i\bz_i^2,\quad,\bz_i\overset{\iid}{\sim}\cN(0,1),
\]
and proposed two ways to approximate the asymptotic distribution. We refer to their paper for the details.

\section{Do minority neighborhoods pay higher insurance premiums?}
\label{sec:insurance}

It has been observed that drivers from predominantly minority zip codes are often charged higher insurance premiums than drivers from non-minority zip codes \citep{Feltner2015High}. The insurance industry has justified the higher premiums by arguing that drivers from minority neighborhoods have higher risk of accidents. In this section, we examine the claim of the insurance industry using the proposed framework and the data obtained by \cite{Larson2017how}.

Before presenting the results, we briefly describe the data, which was obtained by \cite{Larson2017how} from Quadrant Information Services and S\&P Global Inc. It consists of 98,441 insurance quotes for drivers fitting a single profile: A 30-year-old female teacher with a bachelor's degree, excellent credit, no accidents or moving violations, and who is purchasing a policy for \$100,000 of property damage coverage and \$100,000 to cover medical bills per person up to \$300,000 per accident for the first time. She drives a 2016 Toyota Camry, has a 15 mile daily commute, and drives 13,000 miles a year. The quotes are associated with the zip code of the driver, and by fixing the profile and letting zip code change, we control for factors outside of geography.

The risk of drivers in a zip code is measured by the ratio of dollars paid out for liability claims to the number of insured cars. In California, this ratio is called {\bf average loss} and is a measure of the cost to the insurer of insuring a car in a zip code. Ideally, we would have data on the claims from drivers that fit aforementioned profile, but, unfortunately, we do not have such fine-grained data. We refer to \cite{Larson2017how} for further details regarding the data.

We tested two hypotheses in California: the quotes were independent of the percent minority population given the risk in the associated zip codes ($H_1$); the quotes were independent of whether the associated zip code is underserved given the risk ($H_2$). The California Department of Insurance defines ``underserved'' zip codes as zip codes where (i) the fraction of uninsured drivers exceeds the statewide average by at least 10\%, (ii) the per capita income is below the statewide median, (iii) minorities are at least two-thirds of the population. Among the 1,648 Californian zip codes recorded in the data, there are 145 such underserved zip codes. The results are reported in Table \ref{tab:california}.

\begin{table}
\caption{Table of $p$-values for 9 major insurance companies in California}
\label{tab:california}
\begin{tabular}{lrr} \toprule
Company  &  $H_1$ $p$-value & $H_2$ $p$-value \\ \midrule
Allstate & $< 0.001$ & $< 0.001$ \\
Berkshire Hathaway &    0.488 &    0.456 \\
Farmers & $< 0.001$ & $< 0.001$ \\
Liberty Mutual & $< 0.001$ & $< 0.001$ \\
Mercury &    0.925 &    0.648 \\
Nationwide & $< 0.001$ & $< 0.001$ \\
Progressive &  $< 0.001$ &    0.365 \\
State Farm & $< 0.001$ & $< 0.001$ \\
USAA & $< 0.001$ & $< 0.001$ \\
\bottomrule
\end{tabular}
\end{table}

We examine the data on Progressive Group more closely because there is a discrepancy between the results of the test of $H_1$ (its quotes are independent of the percent minority population given the risk) and that of $H_2$ (its quotes are independent of whether the associated zip code is underserved given the risk). To comprehend this discrepancy, we redefine underserved zip codes as zip codes where the percent minorities population is at least $q$ for various values of $q$ and test $H_2$ again. The results are reported in Table \ref{tab:progressive}. We see that although the percentage of minority population, as a continuous variable, does not pass the conditional independence test at $0.05$ level, the minority indicator derived from it sometimes does.

Mathematically, the discrepancy between the tests is unsurprising: $\bx \not\perp \ba$ does not generally imply $f(\bx)\not\perp \by$. However, its practical implication is noteworthy because it exposes one problem of thresholding a continuous protected attribute. Thresholding tolerates discrimination within subgroups (discrimination within the minority/non-minority subgroup), as long as there is no discrimination across different subgroups. We also note that the $p$-value when $q=0.6$ in \ref{tab:progressive} is quite different from that for $H_1$ in \ref{tab:california}. This is because only $33\%$ of the zip codes where minority percentage exceeds $60\%$ are truly underserved.

\begin{table}
\caption{Table of $p$-values for Progressive Group}
\label{tab:progressive}
\begin{tabular}{lr} \toprule
Threshold (q) &  $p$-value \\ \midrule
0.1 & 0.297 \\
0.2 & 0.778 \\
0.3 & 0.698 \\
0.4 & 0.002 \\
0.5 & $< 0.001$ \\
0.6 & 0.010 \\
0.7 & 0.110 \\
0.8 & 0.061 \\
0.9 & 0.009 \\
\bottomrule
\end{tabular}
\end{table}

In Illinois and Texas, we tested the hypothesis that the quotes were independent of the percent minority population given the risk in the associated zip codes. In Illinois, we excluded the zip codes in Chicago because Chicago has a law that require insurers to charge the same price for bodily injury insurance. The results are reported in Table \ref{tab:il-tx}.

\begin{table}
\caption{Table of $p$-values for 9 major insurance companies in Illinois and Texas}
\label{tab:il-tx}
\begin{tabular}{lrr} \toprule
Insurance group  &  Illinois & Texas \\ \midrule
Allstate & $< 0.001$ & $< 0.001$ \\
American Family & $< 0.001$ & \\
Auto Owners & $< 0.001$ & \\
Berkshire Hathaway & $< 0.001$ & $< 0.001$ \\
Country Financial & $< 0.001$ & \\
Erie & $< 0.001$ & \\
Farmers & $< 0.001$ & $< 0.001$ \\
Hartford Fire & $< 0.001$ & \\
Liberty Mutual & $< 0.001$ & $< 0.001$ \\
Metropolitan & $< 0.001$ & \\
Nationwide & & $< 0.001$ \\
Pekin & $< 0.001$ & \\
Progressive & $< 0.001$ & $< 0.001$ \\
State Farm & $< 0.001$ & $< 0.001$ \\
Traverlers & $< 0.001$ & \\
USAA & $< 0.001$ & $< 0.001$ \\
\bottomrule
\end{tabular}
\end{table}

Finally, 
we apply the randomization procedure described in Section \ref{sec:randomization} to adjust the premium. We concentrate on one insurer (Garrison Property and Casualty Insurance Company) and only on the property damage policy premium. There are two protected group. The white-non-Hispanics and the rest. However, this attribute is not given (to the insurer and in the data) and is derived from the proportions of the the two groups within any zip-code area.

Let $\bx$ be the premium in the zip code, $\bz$ the state defined risk and $\bq$ the proportion of whites (non-Hispanic) within the zip code area.  Linear regression of $\bx$ on $\bz$, $\bq$, and $\bq\cdot\bz$ finds all coefficients to be significant at the 0.001 level. To proceed we make the (clearly unrealistic) assumption that all zip code areas have the same number of car insured  by the discussed insurer. Under this assumption the distribution of the excess premiums paid by the two group (after controling for the risk) is plotted in Figure \ref{fig:randinsured}(a). In this semi-artificial setup, a white customer pays, on the average 1.54 USD less than predicted, while a minority group members pays 1.65 USD more with standard deviations equal to 25.7 and 29.2 USD respectively.

Since the group membership is concealed from the insurer, this cannot be corrected directly. We suggest that in such a situation a cross-subsidization between zip code areas, where the premium would have a component which is proportional to the deviation of $\bq$ from its mean and a randomized component proportional to $\bq$. Practically, randomization can be achieved by making the premium depend slightly on hardly relevant information about the customer or the location. In Figure \ref{fig:randinsured}(b) we present the result of this process. Both the white and non-white groups have the same mean and standard deviation (0.0 and 29.8 USD respectively).
\begin{figure}
  \begin{subfigure}[b]{0.48\textwidth}
    \includegraphics[width=\textwidth]{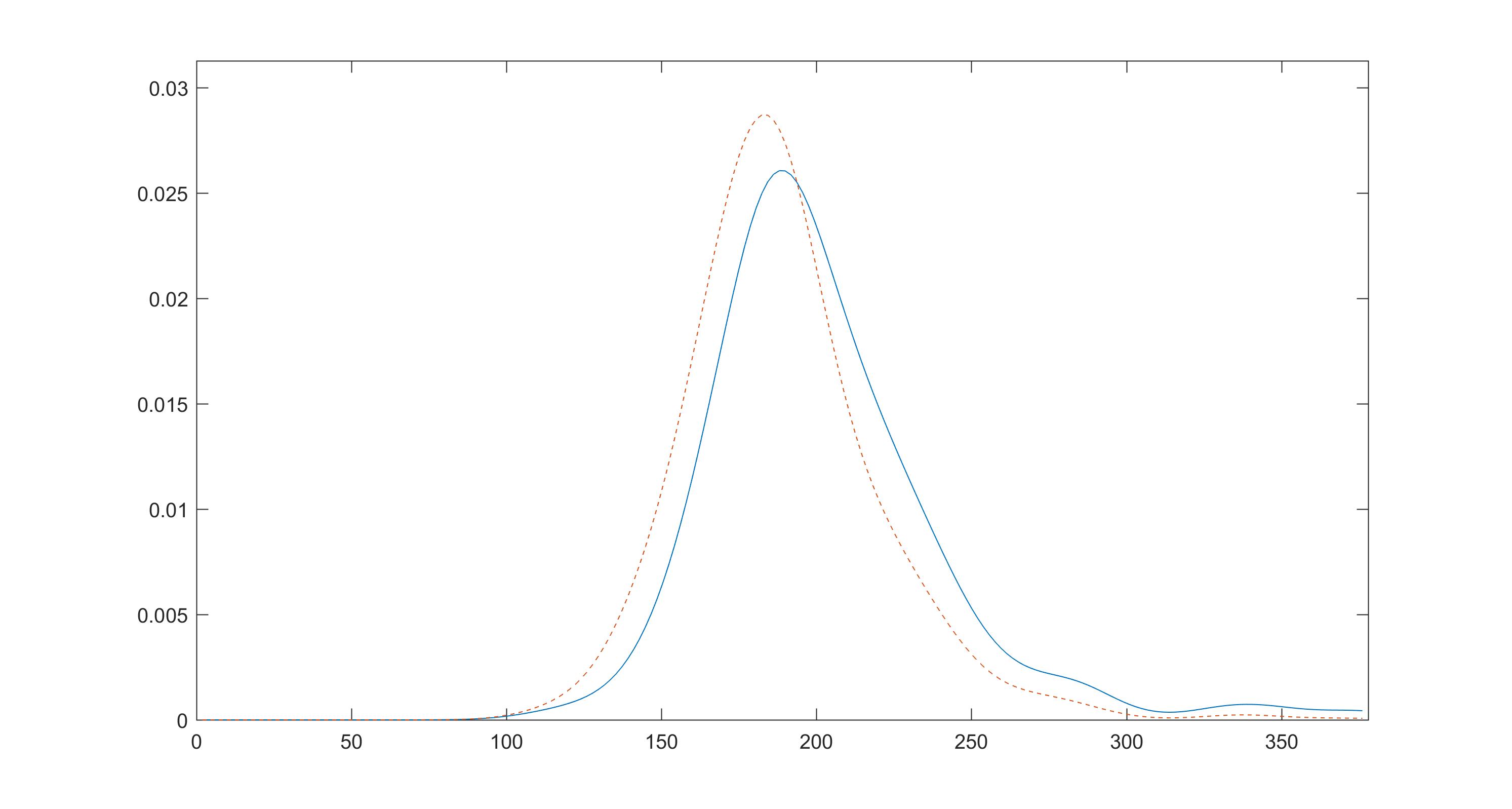}
    \caption{}
  \end{subfigure}
  ~
  \begin{subfigure}[b]{0.48\textwidth}
    \includegraphics[width=\textwidth]{markovkernel500}
    \caption{}
  \end{subfigure}

  \caption{Car insurance premium distribution before and after correction. The red broken line is for the white-non-Hispanic group. The blue full line are the rest.}
  \label{fig:randinsured}
\end{figure}

\section{Summary and discussion}

We identified conditional parity as a general notion of non-discrimination in machine learning. It formalizes the implicit comparison in claims of discrimination and is applicable beyond supervised learning. It also includes many recently proposed notions of non-discrimination, including counterfactual ones.

The main takeaway for practitioners is the necessity of specifying not only protected attribute but also discriminatory ones in any rigorous notion of non-discrimination. Ignoring the discriminatory attributes may lead to ambiguous definitions. Consider the recent debate on whether no sex-based discrimination implies no discrimination based on sexual orientation \citep{Thayer2017Hively}. In {\it Hively v. Ivy Tech}, the majority opinion expressed ``common-sense reality [is] that it is actually impossible to discriminate on the basis of sexual orientation without discriminating on the basis of sex''. However, by letting the gender to which one is attracted to be the discriminatory attribute, we see that it is indeed possible to discriminate on sexual orientation but not on gender. The ambiguity in the prohibition of sex-based discrimination is in the non-specification of discriminatory attributes.

\begin{example}
Let $\ba$ be gender and $\bz$ be the gender to which one is attracted to. Intuitively, no discrimination based on the gender requires the segments of the population on the same row of Table \ref{tab:gender-sexual-orientation} be treated equally, while no discrimination based on sexual orientation requires the segments on the same column be treated equally.

\begin{table}
  \centering
  \caption{}
  \vspace{6pt}
  \begin{tabular}{l|g|g|}
  \rowcolor{white} \multicolumn{1}{l}{} & \multicolumn{1}{l}{$\ba=$ female} & \multicolumn{1}{l}{$\ba=$ male} \\[12pt] \hhline{~|-|-|}
  \rowcolor{white} $\bz=$ female & \parbox[t]{60pt}{homosexual\\female} & \parbox[t]{60pt}{heterosexual\\male} \\[12pt] \hhline{~|-|-|}
  $\bz=$ male & \parbox[t]{60pt}{heterosexual\\female} & \parbox[t]{60pt}{homosexual\\male} \\[12pt] \hhline{~|-|-|}
  \end{tabular}
  \label{tab:gender-sexual-orientation}
\end{table}
\end{example}

Finally, we mention that CP is amenable to statistical analysis. We studied randomization as a general approach to achieving CP, as well as a kernel-based approach to check whether the output of a black-box machine learning algorithm satisfies CP. Most prior work on non-discrimination in machine learning has focused on designing non-discriminatory machine learning algorithms. However, to enforce non-discrimination, methods to detect violations are crucial, and we look forward to developments in future work.

\appendix
\section{Proof of auxiliary lemmas}

\begin{proof}[Proof of Lemma \ref{lem:hsic-test-stats}]
To keep things simple, we evaluate $\|\hSigma_{x,a\mid z}\|_{\HS}^2$ instead of $\|\hSigma_{(x,z),a\mid z}\|_{\HS}^2$. To obtain an expression of $\|\hSigma_{(x,z),a\mid z}\|_{\HS}^2$, simply replace $\bK_x$ by $\bK_{x,z} = \bK_x\cdot\bK_z$. We also abuse notation and denote linear combinations of the form $\sum_{i=1}^n\alpha_i k_x(\bx_i,\cdot)$ by $X\alpha$, where
\[
\bX = \begin{bmatrix} \mid & & \mid \\ k_x(\bx_1,\cdot) & \dots & k_x(\bx_n,\cdot) \\ \mid & & \mid\end{bmatrix}
\]
is an ``infinite matrix''. In this notation, we have
\[
\begin{aligned}
\hSigma_{x,a} &= \frac1n\sum_{i=1}^n(k_x(\bx_i,\cdot) - \hmu_x)\otimes((k_a(\ba_i,\cdot) - \hmu_a)) \\
&= \frac1n\sum_{i=1}^n(\bX(e_i - {\textstyle\frac1n\ones_n}))\otimes(\bA(e_i - {\textstyle\frac1n\ones_n})) \\
&= \frac1n\sum_{i=1}^n(\bX M_ne_i)\otimes(\bA M_ne_i),
\end{aligned}
\]
where $M_n:= {\textstyle I_n - \frac1n\ones_n\ones_n^T}$. The (squared) HS norm of the $\hSigma_{(x,z),a\mid z}$ is
\begin{equation}
\begin{aligned}
\langle\hSigma_{x,a\mid z},\hSigma_{x,a\mid z}\rangle_{\HS} &= \langle\hSigma_{x,a},\hSigma_{x,a}\rangle_{\HS} - 2\langle\hSigma_{x,y},\hSigma_{x,z}(\hSigma_{z,z} + \lambda I)^{-1}\hSigma_{z,a}\rangle_{\HS} \\
&+ \langle\hSigma_{x,z}(\hSigma_{z,z} + \lambda I)^{-1}\hSigma_{z,a},\hSigma_{x,z}(\hSigma_{z,z} + \lambda I)^{-1}\hSigma_{z,a}\rangle_{\HS}.
\end{aligned}
\label{eq:squared-HS-norm}
\end{equation}

For now, we focus on evaluating the first term. Since $\hSigma_{x,a}$ is a mapping from a subspace $\cS_a$ of $\cH_a$ to an subspace $\cS_x$ of $\cH_x$, we may restrict to the subspaces. Let $\{u_i\}_{i=1}^n$ be an orthonormal basis of $\cS_a$. We have
\[
u_i = \sum_{j=1}^n\alpha_{i,j}(k_a(\ba_i,\cdot) - \hmu_a) = \bA({\textstyle I_n - \frac1n\ones_n\ones_n^T})\alpha_i = \bA M_n\alpha_i,
\]
where $\alpha_i = \begin{bmatrix}\alpha_{i,1} & \dots & \alpha_{i,n}\end{bmatrix}^T$, and
\[
\begin{aligned}
\hSigma_{x,a}u_i &= ({\textstyle\frac1n\sum_{j=1}^n(\bX M_ne_j)\otimes(\bA M_ne_j)})(\bA M_n\alpha_i) \\
&= \frac1n\sum_{j=1}^n((\bX M_ne_j)\otimes(\bA M_ne_j))(\bA M_n\alpha_j) \\
&= \frac1n\sum_{j=1}^n\langle \bA M_ne_j,\bA M_n\alpha_i\rangle_a(\bX M_ne_j).
\end{aligned}
\]
By the reproducing property, we have
\[
\begin{aligned}
\langle \bA\alpha,\bA\beta\rangle_a &= \sum_{i=1}^n\sum_{j=1}^n\alpha_i\beta_j\langle k_a(\ba_i,\cdot),k_a(\ba_j,\cdot)\rangle_a\\
&= \sum_{i=1}^n\sum_{j=1}^n\alpha_i\beta_jk_a(\ba_i,\ba_j) \\
&= \alpha^T\bG_a\beta,
\end{aligned}
\]
where $\bG_a$ is the Gram matrix whose entries are $\bigl[G\bigr]_{i,j} = k(\ba_i,\ba_j)$, for any $\alpha,\beta\in\reals^n$. Thus
\[
\hSigma_{x,a}u_i = \frac1n\sum_{j=1}^n(e_j^TM_n\bG_a M_n\alpha_i)(\bX M_ne_j) = \frac1n\bX M_n\bG_a M_n\alpha_i = \frac1n\bX\bK_a\alpha_i,
\]
where $K_a := M_n\bG_a M_n$ is the centered Gram matrix. The first term in \eqref{eq:squared-HS-norm} is
\[
\begin{aligned}
\langle\hSigma_{x,a},\hSigma_{x,a}\rangle_{\HS} &= \sum_{i=1}^n\langle\hSigma_{x,a}u_i,\hSigma_{x,a}u_i\rangle_x \\
&= \frac{1}{n^2}\sum_{i=1}^n\langle \bX\bK_a\alpha_i,\bX\bK_a\alpha_i\rangle_x, \\
&= \frac{1}{n^2}\alpha_i^T\bK_a\bG_x\bK_a\alpha_i \\
&= \frac{1}{n^2}\alpha_i^T\bK_a\bK_x\bK_a\alpha_i,
\end{aligned}
\]
where the third step is a consequence of the reproducing property and the fourth step is a consequence of $\cR(\bK_a) \subset \cR(M_n)$. Recall $\{u_i\}_{i=1}^n$ is an orthonormal basis of $\cS_a$:
\[
\langle u_i,u_i\rangle_a = \langle \bA M_n\alpha_i, \bA M_n\alpha_j\rangle_a= \alpha_i^T\bK_a\alpha_j = \begin{cases}1 & i=j \\ 0 & i\ne j\end{cases},
\]
We see that $\{\bK_a^{\frac12}\alpha_i\}_{i=1}^n$ is an orthonormal basis of $\reals^n$, which implies
\[
\begin{aligned}
\langle\hSigma_{x,a},\hSigma_{x,a}\rangle_{\HS} &= \frac{1}{n^2}\tr(\bK_a^{\frac12}\bK_x\bK_a^{\frac12}) = \frac{1}{n^2}\tr(\bK_x\bK_a).
\end{aligned}
\]
By similar calculations, it is possible to show that the second and third terms are
\[
\begin{aligned}
&\langle\hSigma_{x,a},\hSigma_{x,z}(\hSigma_{z,z} + \lambda I_z)^{-1}\hSigma_{z,a}\rangle_{\HS} \\
&\quad= \frac{1}{n^2}\tr(\bK_x\bK_z(\bK_z + \lambda M_n)^\dagger\bK_y), \\
&\langle\hSigma_{x,z}(\hSigma_{z,z} + \lambda I_z)^{-1}\hSigma_{z,y},\hSigma_{x,z}(\hSigma_{z,z} + \lambda I_z)^{-1}\hSigma_{z,a}\rangle_{\HS} \\
&\quad= \frac{1}{n^2}\tr((\bK_z + \lambda M_n)^\dagger\bK_z\bK_x\bK_z(\bK_z + \lambda M_n)^\dagger\bK_a).
\end{aligned}
\]
\end{proof}

\bibliographystyle{imsart-nameyear}
\bibliography{yuekai}

\end{document}